\documentclass{article}
\usepackage[letterpaper,margin=1in]{geometry}

\usepackage{xcolor}
\usepackage[pdftitle={Inferring Generative Model Structure with Static Analysis},
            pdfauthor={Paroma Varma, Bryan He, Payal Bajaj, Imon Banerjee, Nishith Khandwala, Daniel L. Rubin, Christopher R\'{e}},
            pdfsubject={},
            pdfkeywords={generative models}]{hyperref}

\usepackage[utf8]{inputenc} 
\usepackage[T1]{fontenc}    
\usepackage{url}            
\usepackage{booktabs}       
\usepackage{amsfonts}       
\usepackage{nicefrac}       
\usepackage{microtype}      
\usepackage{graphicx}
\usepackage{amsmath}
\usepackage{amsthm}
\usepackage{multirow}
\usepackage{cleveref}
\usepackage{tikz}
\usepackage{authblk}
\usepackage{ltxtable}
\usetikzlibrary{shapes.geometric}
\usepackage[numbers]{natbib}

\usepackage{ltxtable}

\usepackage{thmtools} 
\usepackage{thm-restate}

\title{Inferring Generative Model Structure with Static Analysis}
\date{}

\author[ ]{Paroma Varma, Bryan He, Payal Bajaj, Imon Banerjee, Nishith Khandwala, \\ Daniel L. Rubin, Christopher R\'{e}\vspace{-0.2cm}}

\affil[ ]{Stanford University}
\affil[ ]{ \texttt{\{paroma,bryanhe,pabajaj,imonb,nishith,rubin\}@stanford.edu, chrismre@cs.stanford.edu}}

\newcommand{\num}[1]{{#1}} 

\begin{document}

\maketitle

\begin{abstract}
Obtaining enough labeled data to robustly train complex discriminative models is a major bottleneck in the machine learning pipeline.
A popular solution is combining multiple sources of weak supervision using generative models.
The structure of these models affects training label quality, but is difficult to learn without any ground truth labels.
We instead rely on these weak supervision sources having some structure by virtue of being encoded programmatically.
We present Coral, a paradigm that infers generative model structure by statically analyzing the code for these heuristics, thus reducing the data required to learn  structure significantly.
We prove that Coral's sample complexity scales quasilinearly with the number of heuristics and number of relations found, improving over the standard sample complexity, which is exponential in $n$ for identifying $n^\textrm{th}$ degree relations.
Experimentally, Coral matches or outperforms traditional structure learning approaches by up to \num{3.81} F1 points. 
Using Coral to model dependencies instead of assuming independence results in better performance than a fully supervised model by \num{3.07} accuracy points when heuristics are used to label radiology data without ground truth labels.
\end{abstract}

\section{Introduction}
Deep neural networks and other complex discriminative models rely on a large amount of labeled training data for their success.
For many real-world applications, obtaining this magnitude of \emph{labeled} data is one of the most expensive and time consuming aspects of the machine learning pipeline.
Recently, generative models have been used to create training labels from various weak supervision sources, such as heuristics or knowledge bases, by modeling the true class label as a latent variable.
After the necessary parameters for the generative models are learned using unlabeled data, the distribution over the true labels can be inferred.
Properly specifying the structure of these generative models is essential in estimating the accuracy of the supervision sources. 
While traditional structure learning approaches have focused on the supervised case \cite{ravikumar2010high, meinshausen2006high, zhao2006model}, previous works related to weak supervision assume that the structure is user-specified \cite{alfonseca2012pattern, takamatsu2012reducing, roth2013combining, ratner2016data}.
Recently, \citet{bach2017learning} showed that it is possible to learn the structure of these models with a sample complexity that scales sublinearly with the number of possible binary dependencies.
However, the sample complexity scales exponentially for higher degree dependencies, limiting its ability to learn complex dependency structures.
Moreover, the time required to learn the dependencies also grows exponentially with the degree of dependencies, hindering the development of user-defined heuristics.

This poses a problem in many domains, where high degree dependencies are common among heuristics that operate over a shared set of inputs.
These inputs are interpretable characteristics extracted from the data.
For example, various approaches in computer vision use bounding box and segmentation attributes \cite{karpathy2015deep, redmon2016you, kang2017}, like location and size, to weakly supervise more complex image-based learning tasks \cite{dai2015boxsup, xia2013semantic,blaschko2010simultaneous, oquab2015object, branson2011strong}.
Another example comes from the medical imaging domain, where attributes include characteristics such as the area, intensity and perimeter of a tumor, as shown in Figure~\ref{fig:overview}.
Note that these attributes and the heuristics written over them are encoded programmatically.
There typically is a relatively small set of interpretable characteristics, so the heuristics often share these attributes.
This results in high order dependency structures among these sources, which are crucial to model in the generative model that learns accuracies for these sources.

To efficiently learn higher order dependencies, we present Coral, a paradigm that statically analyzes the source code of the weak supervision sources to infer, rather than learn, the complex relations among heuristics.
Coral's sample complexity scales quasilinearly with the number of relevant dependencies. 
Moreover, the time to identify these relations is constant in the degree of dependencies, since it only requires looking at the source code for each heuristic.
Specifically, Coral analyzes the code used to generate the weak supervision heuristics and its inputs to find which heuristics share the same inputs. 
This information is used to generate a dependency structure for the heuristics, and a generative model learns the proper weights for this structure to assign probabilistic labels to training data.

We experimentally validate the performance of Coral across various domains and show that it outperforms traditional structure learning under various conditions while being significantly more computationally efficient.
We show how modeling dependencies leads to an improvement of \num{3.81} F1 points compared to standard structure learning approaches. 
Additionally, we show that Coral can assign labels to data that has no ground truth labels, and this augmented training set results in improving the discriminative model performance by \num{3.07} points.
For a complex relation-based image classification task, 6 heuristic functions written using \emph{only} object label and location as primitives are able to train a model that comes within \num{0.74} points of the F1 score achieved by a fully-supervised model trained on the rich, hand-labeled attribute and relation information in the Visual Genome database \cite{krishna2016visual}.


\section{The Coral Paradigm}
The Coral paradigm takes as input a set of domain-specific primitives and a set of programmatic user-defined heuristic functions that operate over the primitives.
We formally define these abstractions in Section \ref{subsec: abstractions}.
Coral runs static analysis on the source code that creates the primitives and the heuristic functions to identify which sets of heuristics are  related by virtue of sharing primitives (Section~\ref{subsec:sa}).
Once Coral identifies these dependencies, it uses a factor graph to model the relationship between the heuristics, primitives and the true class label. 
We describe the conditions under which Coral can learn the structure of the generative model with significantly less data than traditional approaches in Section~\ref{subsec: model} and show how this affects generative model accuracy via simulations.
Finally, we discuss how Coral learns the accuracies of each heuristic and outputs probabilistic labels for the training data (Section~\ref{subsec: learning}).

\begin{figure}[htbp]
  \centering
  \includegraphics[scale=0.35]{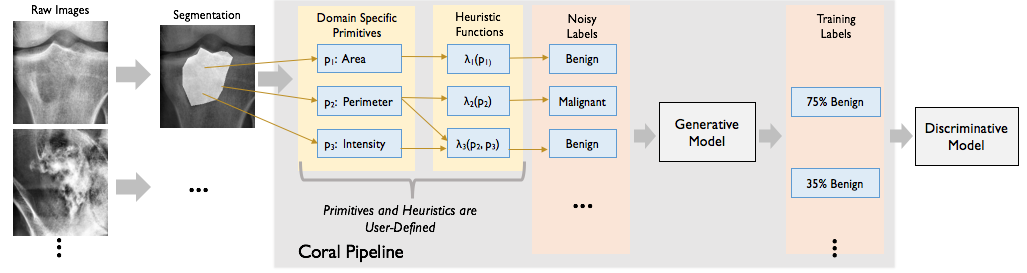}
  \caption{Running example for Coral paradigm. Users apply standard algorithms to segment tumors from the X-ray and extract the \emph{domain-specific primitives} from the image and segmentation. They write \emph{heuristic functions} over the primitives that output a noisy label for each image. The generative model takes these as inputs and provides probabilistic training labels for the discriminative model.}
  \label{fig:overview}
  \setlength{\belowcaptionskip}{-10pt}
\end{figure}

\subsection{Coral Abstractions}
\label{subsec: abstractions}

\paragraph{Domain-Specific Primitives}
Domain-specific primitives (DSPs) in Coral are the simplest elements that heuristic functions take as input and operate over.
DSPs in Coral have semantic meaning, making them interpretable for users. 
This is akin to the concept of language primitives in programming languages, where they are the smallest unit of processing with meaning.
The motivation for making the DSPs domain-specific instead of a general construct for the various data modalities is to allow users to take advantage of existing work in their field, which extracts meaningful characteristics from the raw data.

Figure~\ref{fig:overview} shows an example of a pipeline for bone tumor classification as malignant or benign, inspired by one of our real experiments. 
First, an automated segmentation algorithm is used to generate a binary mask for the tumor \cite{yi20163, kaus2001automated, oliver2010review, sharma2010automated}.
Then, we extract three DSPs from the segmentation: area($p_1$), perimeter($p_2$) and total intensity($p_3$) of the segmented area.
More complex characteristics such as texture, shape and edge features can also be used  \cite{haralick1973textural, banerjee2016computerized, kurtz2014combining} (see Appendix).

We now define a formal construct for programmatically encoding DSPs.
Users generate DSPs in Coral through a primitive specifier function, such as  \texttt{create\_primitives} in Figure~\ref{fig:dsp}(a).
Specifically, this function takes as input a single unlabeled data point (and necessary intermediate representations such as the segmentation) and returns an instance of \texttt{PrimitiveSet}, which maps primitive names to primitive values, like integers (we refer to a specific instance of this class as \texttt{P}). 
Note that \texttt{P.ratio} is composed of two other primitives, while the rest are generated independently from the image and segmentation.

\begin{figure}[htbp]
  \centering
  \includegraphics[scale=0.38]{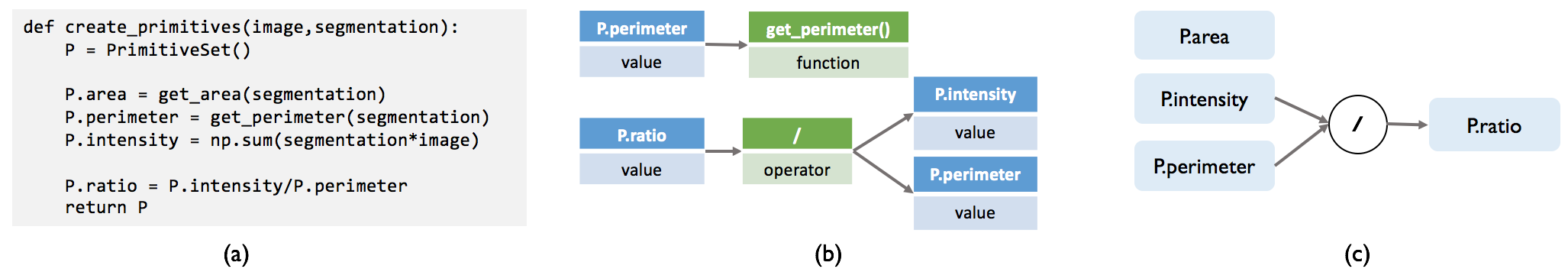}
  \caption{(a) The \texttt{create\_primitives} function that generates primitives. (b) Part of the AST for the \texttt{create\_primitives} function. (c) The composition structure that results from traversing the AST.}
  \label{fig:dsp}
  \setlength{\belowcaptionskip}{-10pt}
\end{figure}

\paragraph{Heuristic Functions}
In Coral, heuristic functions (HFs) can be viewed as a mapping from a subset of the DSPs to a noisy label for the training data, as shown in Figure~\ref{fig:overview}.
In our experience with user-defined HFs, we note that they are mostly nested if-then statements where each statement checks whether the value of a single primitive or a combination of them are above or below a user-set threshold (see Appendix).
As shown in Figure~\ref{fig:fg}(a), they take as input fields of the object \texttt{P} and return a label based on the value of the input primitives.
While our running example focuses on a single data point for DSP generation and HF construction, both procedures are applied to the entire training set to assign a set of noisy labels from each HF to each data point.

\subsection{Static Dependency Analysis}
\label{subsec:sa}
Since the number of DSPs in some domains are relatively small, multiple HFs can operate over the same DSPs, as shown with $\lambda_2$ and $\lambda_3$ in Figure~\ref{fig:overview} and Figure~\ref{fig:fg}(a).
HFs that share at least one primitive are trivially related to each other. 
Prior work learns these dependencies using the labels HFs assign to data points and its probability of success scales with the amount of data available \cite{bach2017learning}.
However, it can only learn dependencies among pairs of HFs efficiently, since the data required grows exponentially with the degree of the HF relation.
This in turn limits the complexity of the dependency structure this method can accurately learn and model.

\paragraph{Heuristic Function Inputs}
Coral takes advantage of the fact that users write HFs over a known, finite set of primitives.
It \emph{infers} dependencies that exist among HFs by simply looking at the source code of how the DSPs and HFs are constructed.
This process \emph{requires no data} to successfully learn the dependencies, making it more computationally efficient than standard approaches.
In order to determine whether any set of HFs share at least one DSP, Coral looks at the input for each HF.
Since the HFs only take as input the DSP they operate over, simply grouping HFs by the primitives they share is an efficient approach for recognizing these dependencies.
As shown in our running example, this would result in Coral not recognizing any dependencies among the HFs since the input for all three HFs are different  (Figure~\ref{fig:fg}(a)).
This, however, would be incorrect, since the primitive \texttt{P.ratio} is composed of \texttt{P.perimeter} and \texttt{P.intensity}, which makes $\lambda_2$ and $\lambda_3$ related.
Therefore, along with looking at the primitives that each HF takes as input, it is also essential to model how these primitives are \emph{composed}.

\paragraph{Primitive Compositions}
We use our running example in Figure~\ref{fig:dsp} to explain how Coral gathers information about DSP compositions.
Coral builds an abstract syntax tree (AST) to represent computation the \texttt{create\_primitives} function performs. 
An AST represents operations involving the primitives as a tree, as shown in Figure~\ref{fig:dsp}(b).

To find primitive compositions from the AST, Coral first finds the expressions in the AST that add primitives to \texttt{P} (denoted in the AST as \texttt{P.name}).
Then, for each assignment expression, Coral traverses the subtree rooted at the assignment expression and adds all other encountered primitives as a dependency for \texttt{P.name}.
If no primitives are encountered in the subtree, the primitive is registered as being independent of the rest. 
The composition structure that results from traversing the AST is shown in Figure~\ref{fig:dsp}(c), where \texttt{P.area}, \texttt{P.intensity}, and \texttt{P.perimeter} are independent while \texttt{P.ratio} is a composition.

\paragraph{Heuristic Function Dependency Structure}
 With knowledge of how the DSPs are composed, we return to our original method of looking at the inputs of the HFs. 
 As before, we identify that $\lambda_1$ and $\lambda_2$ use \texttt{P.area} and \texttt{P.perimeter}, respectively. 
However, we now know that $\lambda_3$ uses \texttt{P.ratio}, which is a composition of \texttt{P.intensity} and \texttt{P.perimeter}.
This implies that $\lambda_3$ will be related to any HF that takes either \texttt{P.intensity}, \texttt{P.perimeter}, or both as inputs.
We proceed to build a relational structure among the HFs and DSPs.
As shown in Figure~\ref{fig:fg}(b), this structure shows which \emph{independent} DSPs each HF operates over. 
This relational structure implicitly encodes dependency information about the HFs --- if an edge points from one primitive to $n$ HFs, those $n$ HFs are in an $n$-way relation by virtue of sharing that primitive.
This dependency information can more formally be encoded in a factor graph shown in Figure~\ref{fig:fg}(c) and discussed in the next section.
Note that we chose a particular programmatic setup for creating DSPs and HFs to explain how static analysis can infer dependencies; however, this process can be modified to work with other setups that encode DSPs and HFs as well.

\begin{figure}[htbp]
  \centering
  \includegraphics[scale=0.30]{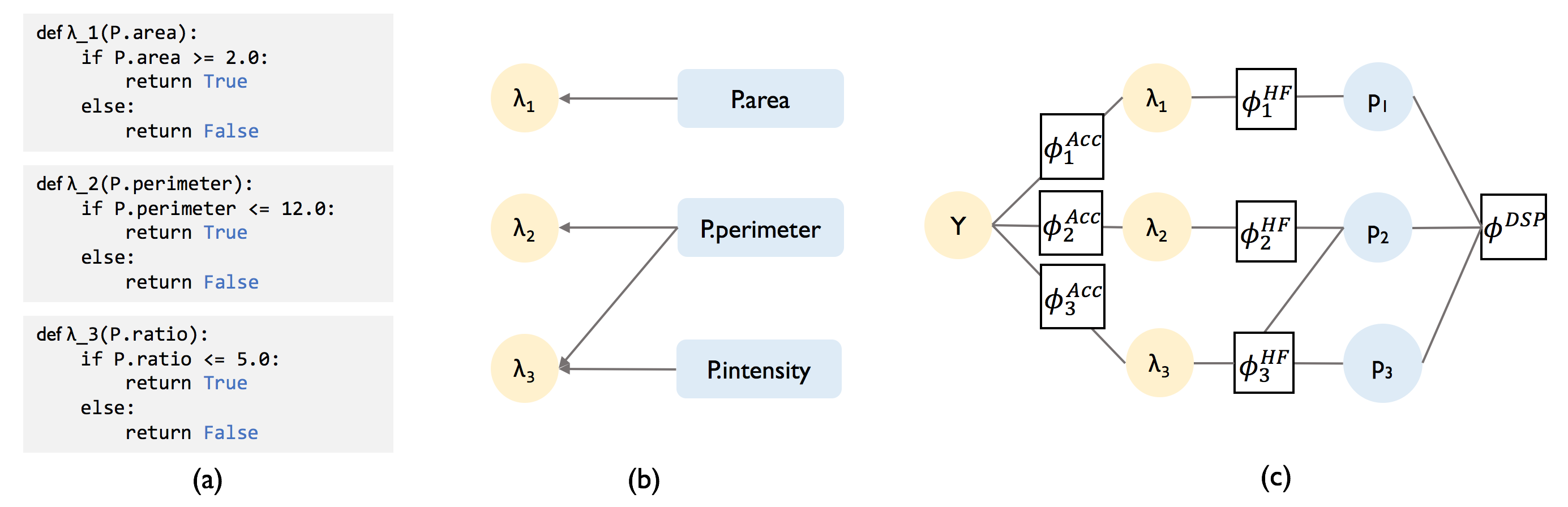}
  \caption{(a) shows the encoded HFs. (b) shows the HF dependency structure where DSP nodes have an edge going to the HFs that use them as inputs (explicitly or implicitly). (c) shows the factor graph Coral uses to model the relationship between HFs, DSPs, and latent class label Y.}
  \label{fig:fg}
\end{figure}

\subsection{Creating the Generative Model}
\label{subsec: model}
We now describe the generative model used to predict the true class labels.
The Coral model uses a factor graph (Figure~\ref{fig:fg}(c)) to model the relationship between the primitives ($p$), heuristic functions ($\lambda$), and latent class label ($Y \in \{-1,1\}$).
We show that by incorporating information about how primitives are shared across HFs from static analysis, this factor graph infers all dependencies between the heuristics that are guaranteed to be present.
In the next section, we describe how Coral recovers additional dependencies between the heuristics by studying empirical relationships between the primitives.

\paragraph{Modeling Heuristic Function Dependencies}
With dependencies inferred via static analysis, our goal is to learn the accuracies for each HF and assign labels to training data accordingly.
The factor graph thus consists of two types of factors: accuracy factors $\phi^{\textrm{Acc}}$ and HF factors from static analysis  $\phi^{\textrm{HF}}$.

The accuracy factors specify the accuracy of each heuristic function $\lambda_i$ and are defined as
\begin{align*}
  \phi^{\textrm{Acc}}_i(Y, \lambda_i) = Y\lambda_i,\ i = 1,\ldots, n
\end{align*}
where $n$ is the total number of heuristic functions. 

The static analysis factors ensure that the heuristics are correctly evaluated based on the HF dependencies found in the previous section.
They ensure that a probability of zero is given to any configuration where an HF does not have the correct value given the primitives it depends on.
The static analysis factors are defined as
\begin{align*}
  \phi^{\textrm{HF}}_i(\lambda_i, p_1, \ldots, p_m) = 
  \begin{cases}
    0 & \textrm{if $\lambda_i$ is valid given $p_1,\ldots, p_m$} \\
    -\infty & \textrm{otherwise}
  \end{cases},\ i = 1,\ldots, n
\end{align*}
Since these factors are obtained directly from static analysis, \emph{they can be recovered with no data}.

However, we note that static analysis is not sufficient to capture all dependencies required in the factor graph to accurately model the process of generating training labels.
Specifically, static analysis can 
\vspace{-5pt}
\begin{itemize}
\setlength\itemsep{0em}
\item[(i)] pick up spurious dependencies among HFs that are not truly dependent on each other, or 
\item[(ii)] miss key dependencies among HFs that exist due to dependencies among the DSPs in the HFs. 
\end{itemize}
\vspace{-5pt}
(i) can occur if some $\lambda_A$ takes as input DSPs $p_i, p_{j}$ and  $\lambda_B$ takes as input DSPs $p_i, p_{k}$, but $p_i$ always has the same value.
Although static analysis would pick up that  $\lambda_A$ and  $\lambda_B$ share a primitive and should have a dependency, this may not be true if $p_{j}$ and $p_{k}$ are independent.
(ii) can occur if two HFs depend on different primitives, but these primitives happen to always have the same value.
In this case, it is impossible for static analysis to infer the dependency between the HFs if the primitives have different names and are generated independently, as described in Section~\ref{subsec:sa}.
A more realistic example comes from our running example, where we would expect the area and perimeter of the tumor to be related.

To account for both cases, it is necessary to capture the possible dependencies that occur among the DSPs to ensure that the dependencies from static analysis do not misspecify the factor graph.
We introduce a factor to account for additional dependencies among the primitives, $\phi^{\textrm{DSP}}$.
There are many possible choices for this dependency factor, but one simple choice is to model pairwise similarity between the primitives.
This allows the dependency factor to be represented as
\begin{align*}
  \phi^{\textrm{DSP}}(p_1, \ldots, p_m) = \sum_{i<j}\phi^{\textrm{Sim}}_{ij}(p_i, p_j)\textrm{, where }
  \phi^{\textrm{Sim}}_{ij}(p_i, p_j) = \mathbb{I}[p_i = p_j].
\end{align*}

Finally, with three types of factors, the probability distribution specified by the factor graph is
\begin{align*}
  P(y, \lambda_1,\ldots,\lambda_n,p_1,\ldots,p_m)
  = \exp
  \left(
  \sum_{i=1}^{n}\theta^{\textrm{Acc}}_i\phi^{\textrm{Acc}}_i
  +
  \sum_{i=1}^{n}\phi^{\textrm{HF}}_i
  +
  \sum_{i=1}^{m}\sum_{j=i+1}^{m} \theta^{\textrm{Sim}}_{ij}\phi^{\textrm{Sim}}_{ij}
  \right)
\end{align*}
where $\theta^{\textrm{Acc}}$ and $\theta^{\textrm{Sim}}_{ij}$ are weights that specify the strength of factors $\phi^{\textrm{Acc}}$ and $\phi^{\textrm{Sim}}_{ij}$.

\paragraph{Inferring Dependencies without Data}

The HF factors capture all dependencies among the heuristic functions that are not represented by the $\phi^{\textrm{DSP}}$ factor.
The dependencies represented by the $\phi^{\textrm{DSP}}$ factor are precisely the dependencies that cannot be inferred via static analysis due to the fact that this factor depends solely on the content of the primitives.
It is therefore impossible to determine what this factor is without data.
 
While assuming that we have the true $\phi^{\textrm{DSP}}$ seems like strong condition, we find that in real-world experiments, including the $\phi^{\textrm{DSP}}$ factor rarely leads to improvements over the case when we only include the $\phi^{\textrm{Acc}}$ and $\phi^{\textrm{HF}}$ factors.
In some of our experiments (see Section~\ref{sec:exp}), we use bounding box location, size and object labels as domain-specific primitives for image and video querying tasks.
Since these primitives are not correlated, modeling the primitive dependency does not lead to any improvement over just modeling HF dependencies from static analysis.
Moreover, in other experiments in which modeling the relation among primitives helps, we observe relatively small benefits (up to $1.1$ points) above what modeling HF dependencies provides (up to $3.47$ points). 
Therefore, without any data, it is possible to model the most important dependencies among HFs that lead to significant gains over the case where no dependencies are modeled.

\subsection{Generating Probabilistic Training Labels}
\label{subsec: learning}
Given the probability distribution of the factor graph, our goal is to learn the proper weights $\theta^{\textrm{Acc}}_i$ and $\theta^{\textrm{Sim}}_{ij}$.
Coral adopts structure learning approaches described in recent work \cite{ravikumar2010high,bach2017learning} that learns dependency structures in the weak supervision setting and maximizes the $\ell_1$-regularized marginal pseudolikelihood of each primitive to learn the weights of the relevant factors.

To learn the weights of the generative model, we use contrastive divergence \cite{hinton2002training} as a maximum likelihood estimation routine and maximize the marginal likelihood of the observed primitives.
Gibbs sampling is used to estimate the intractable gradients, which are then used in stochastic gradient descent.
Because the HFs are typically deterministic functions of the primitives (represented as the $-\infty$ value of the correctness factors for invalid heuristic values), standard Gibbs sampling will not be able to mix properly.
As a result, we modify the Gibbs sampler to simultaneously sample one primitive along with all heuristics that depend on it.
Despite the fact that the true class label is latent, this process still converges to the correct parameter values \cite{ratner2016data}.
Additionally, the amount of data necessary to learn the parameters scales quasilinearly with the number of parameters.
In our case, the number of parameters is just the number of heuristics $n$ and number of relevant primitive similarity dependencies $s$.
We now formally state the conditions for this result, which match those of \citet{ratner2016data}, and provide the sample complexity of our method.
First, we assume that there exists some feasible parameter set $\Theta \subset\mathbb{R}^n$ that is known to contain the parameter $\theta^* = (\theta^{\textrm{Acc}},\ \theta^{\textrm{Sim}})$ that models the true distribution $\pi^*$ of the data:
\begin{align}
  \label{eq: exist}
  \exists \theta^* \in\Theta\textrm{ s.t. }\forall \pi^*(p_1, \ldots, p_m, Y) = \mu_{\theta}(p_1, \ldots, p_m, Y).
\end{align}
Next, we must be able to accurately learn $\theta^*$ if we are provided with labeled samples of the true distribution.
Specifically, there must be an asymptotically unbiased estimator $\hat\theta$ that takes some set of labeled data $T$ independently sampled from $\pi^*$ such that for some $c > 0$,
\begin{align}
  \label{eq: cov}
  \textrm{Cov}\left(\hat \theta(T)\right) \preceq \left(2c|T|\right)^{-1}I.
\end{align}
Finally, we must have enough sufficiently accurate heuristics so that we have a reasonable estimate of Y.
For any two feasible models $\theta_1,\theta_2\in\Theta$,
\begin{align}
  \label{eq: acc}
  \mathbb{E}_{(p_1,\ldots,p_m,Y)\sim \mu_{\theta_1}} \left[\textrm{Var}_{(p'_1,\ldots,p'_m,Y')\sim \mu_{\theta_2}}\left(Y'\mid p_1=p'_1,\ldots,p_m=p'_m\right)\right] \leq \frac{c}{n + s}
\end{align}

\begin{restatable}{proposition}{thmscale}
  \label{thm: scale}
  Suppose
  that we run stochastic gradient descent to produce estimates of the weights $\hat \theta = (\hat \theta^{\textrm{Acc}},\ \hat \theta^{\textrm{Sim}})$ in a setup satisfying conditions \eqref{eq: exist}, \eqref{eq: cov}, and \eqref{eq: acc}.
  Then, for any fixed error $\epsilon > 0$, if the number of unlabeled data points $N$ is at least $\Omega\left[(n + s) \log (n + s)\right]$, then our expected parameter error is bounded by $\mathbb{E}\left[\|\hat \theta - \theta^*\|^2\right] \leq \epsilon^2$.
\end{restatable}
The proof follows from the sample complexity of \citet{ratner2016data} and appears in the appendix.

With the weights $\hat \theta^{\textrm{Acc}}_i$ and $\hat \theta^{\textrm{Sim}}_{ij}$ maximizing the marginal likelihood of the observed primitives,
we have a fully specified factor graph and complete generative model, which can be used to predict the latent class label.
For each data point, we have the domain-specific primitives, from which the heuristic functions can compute noisy labels.
Through the accuracy factors, we then estimate a distribution for the latent class label and use these noisy labels to train a discriminative model.

\begin{figure}[tbhp]
  \centering
  \includegraphics[scale=0.4]{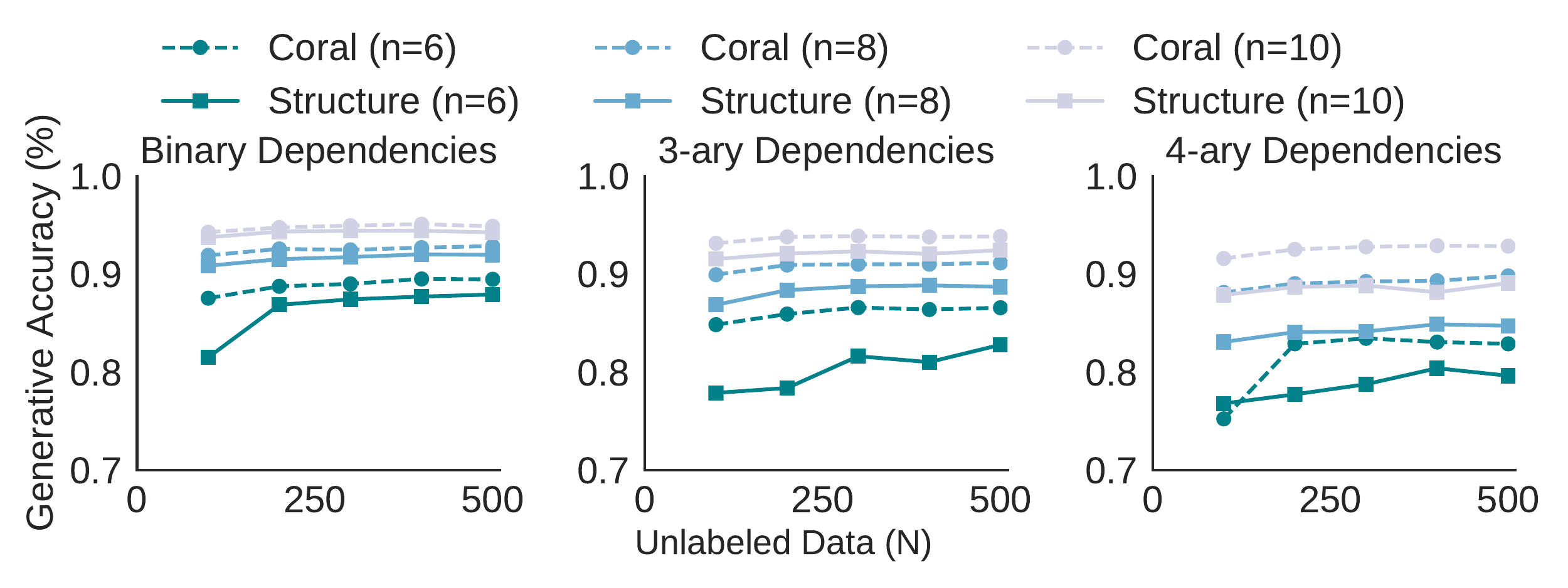}
  \caption{Simulation comparing the generative model accuracy with structure learning and Coral.}
  \label{fig: sim}
\end{figure}

In Figure \ref{fig: sim}, we present a simulation to empirically compare our sample complexity with that of structure learning \cite{bach2017learning}.
In our simulation, we have $n$ HFs, each with an accuracy of $75\%$, and explore settings with where there exists one binary, 3-ary and 4-ary dependency among the HFs. The dependent HFs share exactly one primitive, and the primitives themselves are independent ($s=0$).
As $N$ increases, both methods improve in performance due to improved estimates of the heuristic accuracies and dependencies.
In the case with a binary dependency, structure learning recovers the necessary dependency with few samples, and has similar performance to Coral.
In contrast, in the second and third settings with high-order dependencies, structure learning struggles to recover the relevant dependency, and consistently performs worse than Coral even as more training data is provided.

\section{Experimental Results}
\label{sec:exp}
We seek to experimentally validate the following claims about our approach. 
Our first claim is that HF dependencies inferred via static analysis (HF Dep) improve significantly over a model that does not take dependencies into account.
Second, we compare to a structure learning approach for weak supervision \cite{bach2017learning} and show how we outperform it over a variety of domains. 
Finally, we show that in case there are primitive dependencies, Coral can learn and model those as well (HF+DSP Dep).
Our classification tasks range from specialized medical domains to natural images and video, and we include details of the DSPs and HFs in the Appendix.
We compare our approach to generative models that learn the accuracies of different heuristics, specifically one that assumes the heuristics are independent (Indep), and \citet{bach2017learning} that \emph{learns} the inter-heuristic dependencies (Learn Dep). We also compare to majority vote (MV) and the fully supervised (FS) case, and measure the performance of the discriminative model trained with labels generated using the above methods.

\paragraph{Visual Genome and ActivityNet Classification} 
We explore how to extract complex relations in images and videos given object labels and their bounding boxes.
We used subsets of two datasets, Visual Genome \cite{krishna2016visual} and ActivityNet \cite{caba2015activitynet}, and defined our task as finding images of ``a person biking down a road'' and finding basketball videos, respectively.
For both tasks, a small set of DSPs were shared heavily among HFs, and modeling the dependencies observed by static analysis led to a significant improvement over the independent case.
Since these dependencies involved groups of 3 or more heuristics, Coral improved significantly, by up to 3.81 F1 points, over structure learning as well, which was unable to model these dependencies due to the lack of enough data.
Moreover, modeling primitive dependencies did not help since the primitives were indeed independent (Table~\ref{table:stats}).
We report our results for these tasks in terms of the F1 score, harmonic mean of the precision and recall.

\begin{figure}[htbp]
  \centering
  \includegraphics[scale=0.425]{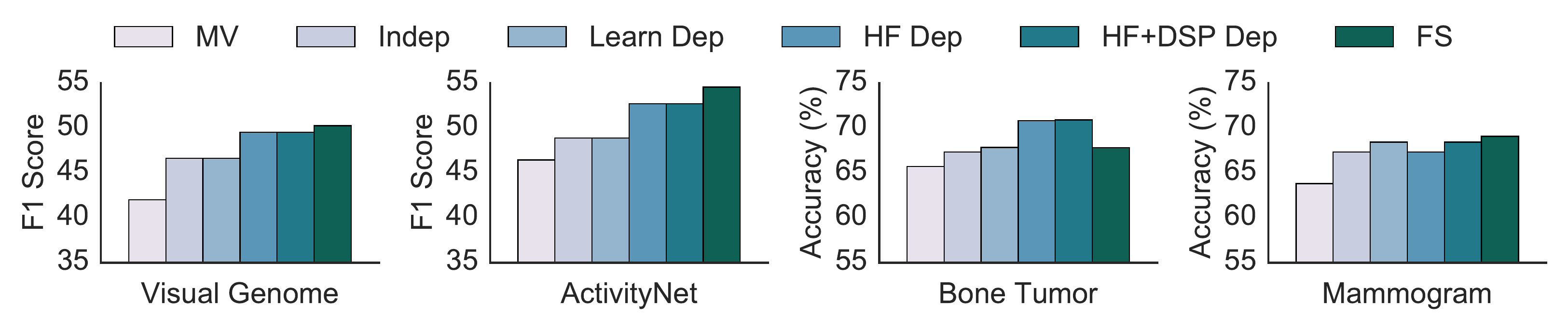}
  \caption{Discriminative model performance comparing HF Dep (HF dependencies from static analysis) and HF+DSP Dep (HF and DSP dependencies) to other methods. Numbers in Appendix.}
  \label{fig:results}
\end{figure}

\begin{table}
\centering
\caption{Heuristic Function (HF) and Domain-Specific Primitive (DSP) statistics. Discriminative model improvement with HF+DSP Dep over other methods. *improvements shown in terms of F1 score, rest in terms of accuracy. ActivityNet model is LR using VGGNet embeddings as features.}
\label{table:stats}
\small
\centering
\begin{tabular}{@{}ccccccccc@{}}
\toprule
\multirow{2}{*}{\textbf{Application}} & \multicolumn{3}{c}{\textbf{Statistics}}                      & \multirow{2}{*}{\textbf{Model}} & \multicolumn{4}{c}{\textbf{Improvement Over}}                   \\ \cmidrule(lr){2-4} \cmidrule(l){6-9} 
                                      & \textbf{\# DSPs} & \textbf{\# HFs} & \textbf{\# Shared DSPs} &                                 & \textbf{MV} & \textbf{Indep} & \textbf{Learn Dep} & \textbf{FS} \\ \midrule
Visual Genome                         & 7                & 5               & 2                       & GoogLeNet                       & 7.49*       & 2.90*          & 2.90*              & -0.74*       \\
ActivityNet                           & 5                & 4               & 2                       & VGGNet+LR                          & 6.23*       & 3.81*          & 3.81*              & -1.87*       \\ \midrule
Bone Tumor                            & 17               & 7               & 0                       & LR                              & 5.17        & 3.57           & 3.06               & 3.07        \\
Mammogram                             & 6                & 6               & 0                       & GoogLeNet                       & 4.62        & 1.11           & 0                  & -0.64       \\ \bottomrule
\end{tabular}
\end{table}

\paragraph{Bone Tumor Classification}
We used a set of 802 labeled bone tumor X-ray images along with their radiologist-drawn segmentations.
We defined our task to differentiate between aggressive and non-aggressive tumors, generated HFs that were a combination of hand-tuned rules and decision-tree generated rules (tuned on a small subset of the dataset).
The discriminative model utilized a set of 400 hand-tuned features (note that there is no overlap between these features and the DSPs) that encoded various shape, texture, edge and intensity-based characteristics.
Although there were no explicitly shared primitives in this dataset, the generative model was still able to model the training labels more accurately with knowledge of \emph{how many primitives} the heuristics operated over, thus improving over the independent case significantly. 
Moreover, a small dataset size hindered structure learning, which gave a minimal boost over the independent case (Table~\ref{table:stats}).
When we used heuristics in Coral to label an additional 800 images that had no ground truth labels, we beat the previous FS score by \num{3.07} points (Figure~\ref{fig:results}, Table~\ref{table:stats}).

\paragraph{Mammogram Tumor Classification}
We used the DDSM-CBIS \cite{ddsm} dataset, which consists of 1800 scanned film mammograms and associated segmentations for the tumors in the form of binary masks.
We defined our task to identify a tumor as malignant or benign. Each heuristic only operated over one primitive, resulting in no dependencies that static analysis could identify. 
In this case, structure learning performed better than Coral when we only used static analysis to infer dependencies (Figure~\ref{fig:results}). 
However, including primitive dependencies allowed us to match structure learning, resulting in a  $\num{1.11}$ point improvement over the independent case (Table~\ref{table:stats}).

\section{Related Work}

As the need for labeled training data grows, a common alternative is to utilize weak supervision sources such as distant supervision \cite{craven1999constructing,mintz2009distant}, multi-instance learning \cite{riedel2010modeling,hoffmann2011knowledge}, and heuristics \cite{bunescu2007learning,shin2015incremental}. 
Specifically for images, weak supervision using object detection and segmentation or visual databases is a popular technique as well (detailed discussion in Appendix).
Estimating the accuracies of these sources without access to ground truth labels is a classic problem \cite{dawid1979maximum}. 
Methods such as crowdsourcing, boosting, co-training, and learning from noisy labels are some of the popular approaches that can combine various sources of weak supervision to assign noisy labels to data (detailed discussion in Appendix).
However, Coral does not require \emph{any} labeled data to model the dependencies among the domain-specific primitives and heuristics, which can be interpreted as workers, classifiers or views for the above methods. 

Recently, generative models have also been used to combine various sources of weak supervision  \cite{alfonseca2012pattern, takamatsu2012reducing,  roth2013combining}.
One specific example, data programming \cite{ratner2016data}, proposes using multiple sources of weak supervision for text data in order to describe a generative model and subsequently learns the accuracies of these source.
Coral also focuses on multiple programmatically encoded heuristics that can weakly label data and learns their accuracies to assign labels to training data.
However, Coral adds an additional layer of domain-specific primitives in its generative model, which allows it to generalize beyond text-based heuristics.
It also \emph{infers} the dependencies among the heuristics and the primitives, rather than requiring users to specify them.
Other previous work also assume that this structure in generative models is user-specified \cite{alfonseca2012pattern, takamatsu2012reducing,  roth2013combining}.
Recently, \citet{bach2017learning} showed that it is possible to learn the dependency structure among sources of weak supervision with a sample complexity that scales sublinearly with the number of possible pairwise dependencies.
Coral instead identifies the dependencies among the heuristic functions by inspecting the content of the programmable functions, therefore relying on significantly less data to learn the generative model structure.
Moreover, Coral can also pick up higher-order dependencies, for which structured learning needs large amounts of data to detect.


%
\section{Conclusion and Future Work}
In this paper, we introduced Coral, a paradigm that models the dependency structure of weak supervision heuristics and systematically combines their outputs to assign probabilistic labels to training data.
We described how Coral takes advantage of the programmatic nature of these heuristics in order to infer dependencies among them via static analysis and requires a sample complexity that is quasilinear in the number of heuristics and relations found.  
We showed how Coral leads to significant improvements in discriminative model accuracy over traditional structure learning approaches across various domains.
Coral scratches the surface of the possible ways weak supervision can borrow from the field of programming languages, especially as they are applied to large magnitudes of data and need to be encoded programmatically. 
We look at a natural extension of treating the process of encoding heuristics as writing functions and hope to further explore the interactions between systematic training set creation and concepts from the programming language field.

\small
\paragraph*{Acknowledgments}
We thank Shoumik Palkar, Stephen Bach, and Sen Wu for their helpful conversations and feedback.
We are grateful to Darvin Yi for his assistance with the DDSM dataset based experiments and associated deep learning models.
We acknowledge the use of the bone tumor dataset annotated by Drs. Christopher Beaulieu and Bao Do and carefully collected over his career by the late Henry H. Jones, M.D. (aka ``Bones Jones''). 

This material is based on research sponsored by Defense Advanced Research Projects Agency (DARPA) under agreement number FA8750-17-2-0095. 
We gratefully acknowledge the support of the DARPA SIMPLEX program under No. N66001-15-C-4043,
DARPA FA8750-12-2-0335 and FA8750-13-2-0039,
DOE 108845, 
the National Science Foundation (NSF) Graduate Research Fellowship under No. DGE-114747,
Joseph W. and Hon Mai Goodman Stanford Graduate Fellowship, 
the Moore Foundation, 
National Institute of Health (NIH) U54EB020405,
the Office of Naval Research (ONR) under awards No. N000141210041 and No. N000141310129, 
the Moore Foundation,
the Okawa Research Grant, 
American Family Insurance,
Accenture,
Toshiba, and Intel.  
This research was supported in part by affiliate members and other supporters of the Stanford DAWN project: Intel, Microsoft, Teradata, and VMware.

The U.S. Government is authorized to reproduce and distribute reprints for Governmental purposes notwithstanding any copyright notation thereon.
The views and conclusions contained herein are those of the authors and should not be interpreted as necessarily representing the official policies or endorsements, either expressed or implied, of DARPA or the U.S. Government.
Any opinions, findings, and conclusions or recommendations expressed in this material are those of the authors and do not necessarily reflect the views of DARPA, AFRL, NSF, NIH, ONR, or the U.S. government.

\newpage
\begin{small}
\setlength{\bibsep}{4.5pt}
\bibliographystyle{abbrvnat}
\bibliography{coral_bib}

\begin{thebibliography}{48}
\providecommand{\natexlab}[1]{#1}
\providecommand{\url}[1]{\texttt{#1}}
\expandafter\ifx\csname urlstyle\endcsname\relax
  \providecommand{\doi}[1]{doi: #1}\else
  \providecommand{\doi}{doi: \begingroup \urlstyle{rm}\Url}\fi

\bibitem[Alfonseca et~al.(2012)Alfonseca, Filippova, Delort, and
  Garrido]{alfonseca2012pattern}
E.~Alfonseca, K.~Filippova, J.-Y. Delort, and G.~Garrido.
\newblock Pattern learning for relation extraction with a hierarchical topic
  model.
\newblock In \emph{Proceedings of the 50th Annual Meeting of the Association
  for Computational Linguistics: Short Papers-Volume 2}, pages 54--59.
  Association for Computational Linguistics, 2012.

\bibitem[Bach et~al.(2017)Bach, He, Ratner, and R{\'e}]{bach2017learning}
S.~H. Bach, B.~He, A.~Ratner, and C.~R{\'e}.
\newblock Learning the structure of generative models without labeled data.
\newblock In \emph{ICML}, 2017.

\bibitem[Balsubramani and Freund(2015)]{balsubramani2015scalable}
A.~Balsubramani and Y.~Freund.
\newblock Scalable semi-supervised aggregation of classifiers.
\newblock In \emph{Advances in Neural Information Processing Systems}, pages
  1351--1359, 2015.

\bibitem[Banerjee et~al.(2016)Banerjee, Hahn, Sonn, Fan, and
  Rubin]{banerjee2016computerized}
I.~Banerjee, L.~Hahn, G.~Sonn, R.~Fan, and D.~L. Rubin.
\newblock Computerized multiparametric mr image analysis for prostate cancer
  aggressiveness-assessment.
\newblock \emph{arXiv preprint arXiv:1612.00408}, 2016.

\bibitem[Barb et~al.(2005)Barb, Shyu, and Sethi]{barb2005knowledge}
A.~S. Barb, C.-R. Shyu, and Y.~P. Sethi.
\newblock Knowledge representation and sharing using visual semantic modeling
  for diagnostic medical image databases.
\newblock \emph{IEEE Transactions on Information Technology in Biomedicine},
  9\penalty0 (4):\penalty0 538--553, 2005.

\bibitem[Blaschko et~al.(2010)Blaschko, Vedaldi, and
  Zisserman]{blaschko2010simultaneous}
M.~Blaschko, A.~Vedaldi, and A.~Zisserman.
\newblock Simultaneous object detection and ranking with weak supervision.
\newblock In \emph{Advances in neural information processing systems}, pages
  235--243, 2010.

\bibitem[Blum and Mitchell(1998)]{blum1998combining}
A.~Blum and T.~Mitchell.
\newblock Combining labeled and unlabeled data with co-training.
\newblock In \emph{Proceedings of the eleventh annual conference on
  Computational learning theory}, pages 92--100. ACM, 1998.

\bibitem[Branson et~al.(2011)Branson, Perona, and Belongie]{branson2011strong}
S.~Branson, P.~Perona, and S.~Belongie.
\newblock Strong supervision from weak annotation: Interactive training of
  deformable part models.
\newblock In \emph{Computer Vision (ICCV), 2011 IEEE International Conference
  on}, pages 1832--1839. IEEE, 2011.

\bibitem[Brooks et~al.(1979)Brooks, Creiner, and Binford]{brooks1979acronym}
R.~A. Brooks, R.~Creiner, and T.~O. Binford.
\newblock The acronym model-based vision system.
\newblock In \emph{Proceedings of the 6th international joint conference on
  Artificial intelligence-Volume 1}, pages 105--113. Morgan Kaufmann Publishers
  Inc., 1979.

\bibitem[Bunescu and Mooney(2007)]{bunescu2007learning}
R.~Bunescu and R.~Mooney.
\newblock Learning to extract relations from the web using minimal supervision.
\newblock In \emph{ACL}, 2007.

\bibitem[Caba~Heilbron et~al.(2015)Caba~Heilbron, Escorcia, Ghanem, and
  Carlos~Niebles]{caba2015activitynet}
F.~Caba~Heilbron, V.~Escorcia, B.~Ghanem, and J.~Carlos~Niebles.
\newblock Activitynet: A large-scale video benchmark for human activity
  understanding.
\newblock In \emph{Proceedings of the IEEE Conference on Computer Vision and
  Pattern Recognition}, pages 961--970, 2015.

\bibitem[Craven et~al.(1999)Craven, Kumlien, et~al.]{craven1999constructing}
M.~Craven, J.~Kumlien, et~al.
\newblock Constructing biological knowledge bases by extracting information
  from text sources.
\newblock In \emph{ISMB}, pages 77--86, 1999.

\bibitem[Dai et~al.(2015)Dai, He, and Sun]{dai2015boxsup}
J.~Dai, K.~He, and J.~Sun.
\newblock Boxsup: Exploiting bounding boxes to supervise convolutional networks
  for semantic segmentation.
\newblock In \emph{Proceedings of the IEEE International Conference on Computer
  Vision}, pages 1635--1643, 2015.

\bibitem[Dalvi et~al.(2013)Dalvi, Dasgupta, Kumar, and
  Rastogi]{dalvi2013aggregating}
N.~Dalvi, A.~Dasgupta, R.~Kumar, and V.~Rastogi.
\newblock Aggregating crowdsourced binary ratings.
\newblock In \emph{Proceedings of the 22nd international conference on World
  Wide Web}, pages 285--294. ACM, 2013.

\bibitem[Dawid and Skene(1979)]{dawid1979maximum}
A.~P. Dawid and A.~M. Skene.
\newblock Maximum likelihood estimation of observer error-rates using the {EM}
  algorithm.
\newblock \emph{Applied statistics}, pages 20--28, 1979.

\bibitem[Fischler and Elschlager(1973)]{fischler1973representation}
M.~A. Fischler and R.~A. Elschlager.
\newblock The representation and matching of pictorial structures.
\newblock \emph{IEEE Transactions on computers}, 100\penalty0 (1):\penalty0
  67--92, 1973.

\bibitem[Fukui et~al.(2016)Fukui, Park, Yang, Rohrbach, Darrell, and
  Rohrbach]{fukui2016multimodal}
A.~Fukui, D.~H. Park, D.~Yang, A.~Rohrbach, T.~Darrell, and M.~Rohrbach.
\newblock Multimodal compact bilinear pooling for visual question answering and
  visual grounding.
\newblock \emph{arXiv preprint arXiv:1606.01847}, 2016.

\bibitem[Haralick et~al.(1973)Haralick, Shanmugam,
  et~al.]{haralick1973textural}
R.~M. Haralick, K.~Shanmugam, et~al.
\newblock Textural features for image classification.
\newblock \emph{IEEE Transactions on systems, man, and cybernetics}, 3\penalty0
  (6):\penalty0 610--621, 1973.

\bibitem[Hearst(1992)]{hearst1992automatic}
M.~A. Hearst.
\newblock Automatic acquisition of hyponyms from large text corpora.
\newblock In \emph{Proceedings of the 14th conference on Computational
  linguistics-Volume 2}, pages 539--545. Association for Computational
  Linguistics, 1992.

\bibitem[Hinton(2002)]{hinton2002training}
G.~E. Hinton.
\newblock Training products of experts by minimizing contrastive divergence.
\newblock \emph{Neural computation}, 14\penalty0 (8):\penalty0 1771--1800,
  2002.

\bibitem[Hoffmann et~al.(2011)Hoffmann, Zhang, Ling, Zettlemoyer, and
  Weld]{hoffmann2011knowledge}
R.~Hoffmann, C.~Zhang, X.~Ling, L.~Zettlemoyer, and D.~S. Weld.
\newblock Knowledge-based weak supervision for information extraction of
  overlapping relations.
\newblock In \emph{Proceedings of the 49th Annual Meeting of the Association
  for Computational Linguistics: Human Language Technologies-Volume 1}, pages
  541--550. Association for Computational Linguistics, 2011.

\bibitem[Joglekar et~al.(2015)Joglekar, Garcia-Molina, and
  Parameswaran]{joglekar2015comprehensive}
M.~Joglekar, H.~Garcia-Molina, and A.~Parameswaran.
\newblock Comprehensive and reliable crowd assessment algorithms.
\newblock In \emph{Data Engineering (ICDE), 2015 IEEE 31st International
  Conference on}, pages 195--206. IEEE, 2015.

\bibitem[Kang et~al.(2017)Kang, Emmons, Abuzaid, Bailis, and Zaharia]{kang2017}
D.~Kang, J.~Emmons, F.~Abuzaid, P.~Bailis, and M.~Zaharia.
\newblock Optimizing deep cnn-based queries over video streams at scale.
\newblock \emph{CoRR}, abs/1703.02529, 2017.
\newblock URL \url{http://arxiv.org/abs/1703.02529}.

\bibitem[Karpathy and Fei-Fei(2015)]{karpathy2015deep}
A.~Karpathy and L.~Fei-Fei.
\newblock Deep visual-semantic alignments for generating image descriptions.
\newblock In \emph{Proceedings of the IEEE Conference on Computer Vision and
  Pattern Recognition}, pages 3128--3137, 2015.

\bibitem[Kaus et~al.(2001)Kaus, Warfield, Nabavi, Black, Jolesz, and
  Kikinis]{kaus2001automated}
M.~R. Kaus, S.~K. Warfield, A.~Nabavi, P.~M. Black, F.~A. Jolesz, and
  R.~Kikinis.
\newblock Automated segmentation of mr images of brain tumors 1.
\newblock \emph{Radiology}, 218\penalty0 (2):\penalty0 586--591, 2001.

\bibitem[Krishna et~al.(2016)Krishna, Zhu, Groth, Johnson, Hata, Kravitz, Chen,
  Kalantidis, Li, Shamma, et~al.]{krishna2016visual}
R.~Krishna, Y.~Zhu, O.~Groth, J.~Johnson, K.~Hata, J.~Kravitz, S.~Chen,
  Y.~Kalantidis, L.-J. Li, D.~A. Shamma, et~al.
\newblock Visual genome: Connecting language and vision using crowdsourced
  dense image annotations.
\newblock \emph{arXiv preprint arXiv:1602.07332}, 2016.

\bibitem[Kurtz et~al.(2014)Kurtz, Depeursinge, Napel, Beaulieu, and
  Rubin]{kurtz2014combining}
C.~Kurtz, A.~Depeursinge, S.~Napel, C.~F. Beaulieu, and D.~L. Rubin.
\newblock On combining image-based and ontological semantic dissimilarities for
  medical image retrieval applications.
\newblock \emph{Medical image analysis}, 18\penalty0 (7):\penalty0 1082--1100,
  2014.

\bibitem[Lu et~al.(2016)Lu, Krishna, Bernstein, and Fei-Fei]{lu2016visual}
C.~Lu, R.~Krishna, M.~Bernstein, and L.~Fei-Fei.
\newblock Visual relationship detection with language priors.
\newblock In \emph{European Conference on Computer Vision}, pages 852--869.
  Springer, 2016.

\bibitem[Meinshausen and B{\"u}hlmann(2006)]{meinshausen2006high}
N.~Meinshausen and P.~B{\"u}hlmann.
\newblock High-dimensional graphs and variable selection with the lasso.
\newblock \emph{The annals of statistics}, pages 1436--1462, 2006.

\bibitem[Mintz et~al.(2009)Mintz, Bills, Snow, and Jurafsky]{mintz2009distant}
M.~Mintz, S.~Bills, R.~Snow, and D.~Jurafsky.
\newblock Distant supervision for relation extraction without labeled data.
\newblock In \emph{Proceedings of the Joint Conference of the 47th Annual
  Meeting of the ACL and the 4th International Joint Conference on Natural
  Language Processing of the AFNLP: Volume 2-Volume 2}, pages 1003--1011.
  Association for Computational Linguistics, 2009.

\bibitem[Oliver et~al.(2010)Oliver, Freixenet, Marti, P{\'e}rez, Pont, Denton,
  and Zwiggelaar]{oliver2010review}
A.~Oliver, J.~Freixenet, J.~Marti, E.~P{\'e}rez, J.~Pont, E.~R. Denton, and
  R.~Zwiggelaar.
\newblock A review of automatic mass detection and segmentation in mammographic
  images.
\newblock \emph{Medical image analysis}, 14\penalty0 (2):\penalty0 87--110,
  2010.

\bibitem[Oquab et~al.(2015)Oquab, Bottou, Laptev, and Sivic]{oquab2015object}
M.~Oquab, L.~Bottou, I.~Laptev, and J.~Sivic.
\newblock Is object localization for free? - {W}eakly-supervised learning with
  convolutional neural networks.
\newblock In \emph{Proceedings of the IEEE Conference on Computer Vision and
  Pattern Recognition}, pages 685--694, 2015.

\bibitem[Raicu et~al.(2009)Raicu, Varutbangkul, Furst, and
  Armato~III]{raicu2009modelling}
D.~S. Raicu, E.~Varutbangkul, J.~D. Furst, and S.~G. Armato~III.
\newblock Modelling semantics from image data: opportunities from lidc.
\newblock \emph{International Journal of Biomedical Engineering and
  Technology}, 3\penalty0 (1-2):\penalty0 83--113, 2009.

\bibitem[Ratner et~al.(2016)Ratner, De~Sa, Wu, Selsam, and
  R{\'e}]{ratner2016data}
A.~J. Ratner, C.~M. De~Sa, S.~Wu, D.~Selsam, and C.~R{\'e}.
\newblock Data programming: Creating large training sets, quickly.
\newblock In \emph{Advances in Neural Information Processing Systems}, pages
  3567--3575, 2016.

\bibitem[Ravikumar et~al.(2010)Ravikumar, Wainwright, Lafferty,
  et~al.]{ravikumar2010high}
P.~Ravikumar, M.~J. Wainwright, J.~D. Lafferty, et~al.
\newblock High-dimensional ising model selection using l1-regularized logistic
  regression.
\newblock \emph{The Annals of Statistics}, 38\penalty0 (3):\penalty0
  1287--1319, 2010.

\bibitem[Redmon et~al.(2016)Redmon, Divvala, Girshick, and
  Farhadi]{redmon2016you}
J.~Redmon, S.~Divvala, R.~Girshick, and A.~Farhadi.
\newblock You only look once: Unified, real-time object detection.
\newblock In \emph{Proceedings of the IEEE Conference on Computer Vision and
  Pattern Recognition}, pages 779--788, 2016.

\bibitem[Riedel et~al.(2010)Riedel, Yao, and McCallum]{riedel2010modeling}
S.~Riedel, L.~Yao, and A.~McCallum.
\newblock Modeling relations and their mentions without labeled text.
\newblock In \emph{Joint European Conference on Machine Learning and Knowledge
  Discovery in Databases}, pages 148--163. Springer, 2010.

\bibitem[Roth and Klakow(2013)]{roth2013combining}
B.~Roth and D.~Klakow.
\newblock Combining generative and discriminative model scores for distant
  supervision.
\newblock In \emph{EMNLP}, pages 24--29, 2013.

\bibitem[Sawyer-Lee et~al.(2016)Sawyer-Lee, Gimenez, Hoogi, and Rubin]{ddsm}
R.~Sawyer-Lee, F.~Gimenez, A.~Hoogi, and D.~Rubin.
\newblock Curated breast imaging subset of ddsm, 2016.

\bibitem[Schapire and Freund(2012)]{schapire2012boosting}
R.~E. Schapire and Y.~Freund.
\newblock \emph{Boosting: Foundations and algorithms}.
\newblock MIT press, 2012.

\bibitem[Sharma et~al.(2010)Sharma, Aggarwal, et~al.]{sharma2010automated}
N.~Sharma, L.~M. Aggarwal, et~al.
\newblock Automated medical image segmentation techniques.
\newblock \emph{Journal of medical physics}, 35\penalty0 (1):\penalty0 3, 2010.

\bibitem[Shin et~al.(2015)Shin, Wu, Wang, De~Sa, Zhang, and
  R{\'e}]{shin2015incremental}
J.~Shin, S.~Wu, F.~Wang, C.~De~Sa, C.~Zhang, and C.~R{\'e}.
\newblock Incremental knowledge base construction using {DeepDive}.
\newblock \emph{Proceedings of the VLDB Endowment}, 8\penalty0 (11):\penalty0
  1310--1321, 2015.

\bibitem[Stewart and Ermon(2017)]{stewart2017label}
R.~Stewart and S.~Ermon.
\newblock Label-free supervision of neural networks with physics and domain
  knowledge.
\newblock In \emph{AAAI}, 2017.

\bibitem[Takamatsu et~al.(2012)Takamatsu, Sato, and
  Nakagawa]{takamatsu2012reducing}
S.~Takamatsu, I.~Sato, and H.~Nakagawa.
\newblock Reducing wrong labels in distant supervision for relation extraction.
\newblock In \emph{Proceedings of the 50th Annual Meeting of the Association
  for Computational Linguistics: Long Papers-Volume 1}, pages 721--729.
  Association for Computational Linguistics, 2012.

\bibitem[Xia et~al.(2013)Xia, Domokos, Dong, Cheong, and Yan]{xia2013semantic}
W.~Xia, C.~Domokos, J.~Dong, L.-F. Cheong, and S.~Yan.
\newblock Semantic segmentation without annotating segments.
\newblock In \emph{Proceedings of the IEEE International Conference on Computer
  Vision}, pages 2176--2183, 2013.

\bibitem[Yi et~al.(2016)Yi, Zhou, Chen, and Gevaert]{yi20163}
D.~Yi, M.~Zhou, Z.~Chen, and O.~Gevaert.
\newblock 3-d convolutional neural networks for glioblastoma segmentation.
\newblock \emph{arXiv preprint arXiv:1611.04534}, 2016.

\bibitem[Zhang et~al.(2016)Zhang, Chen, Zhou, and Jordan]{zhang2016spectral}
Y.~Zhang, X.~Chen, D.~Zhou, and M.~I. Jordan.
\newblock Spectral methods meet em: A provably optimal algorithm for
  crowdsourcing.
\newblock \emph{Journal of Machine Learning Research}, 17\penalty0
  (102):\penalty0 1--44, 2016.

\bibitem[Zhao and Yu(2006)]{zhao2006model}
P.~Zhao and B.~Yu.
\newblock On model selection consistency of lasso.
\newblock \emph{Journal of Machine learning research}, 7\penalty0
  (Nov):\penalty0 2541--2563, 2006.

\end{thebibliography}
\end{small}

\newpage
\appendix
\section{Extended Related Works}
We provide additional details here of how writing heuristics over primitives is a popular technique, making Coral widely applicable to existing methods as well. 
We then discuss methods that are related to Coral in combining various sources of weak supervision to generate noisy labels for the data at hand. 
\paragraph{Weak Supervision over Primitives}
A popular approach for creating training sets is to provide weak or distant supervision to label data based on  information from a knowledge base \cite{craven1999constructing, mintz2009distant}, crowdsourcing \cite{dawid1979maximum}, heuristic patterns \cite{bunescu2007learning,hearst1992automatic}, user input \cite{shin2015incremental}, a set of user-defined labeling functions \cite{ratner2016data}, or hand-engineered constraints \cite{stewart2017label}.
Our inspiration for Coral came from observing various weak supervision and image description techniques developed in the field of computer vision.
Early work  looked at describing images in terms of a set of primitives to find instances of described objects in images \cite{fischler1973representation, brooks1979acronym}. 
More recently, learned characteristics such as bounding boxes from object detection and areas from image segmentation have been used in order to weakly supervise more complex image-based learning tasks \cite{dai2015boxsup, xia2013semantic,blaschko2010simultaneous, oquab2015object, branson2011strong}.
Moreover, the recent development of a `knowledge base' for images, Visual Genome \cite{krishna2016visual}, has provided access to a image database with rich, crowdsourced attribute and relational information.
This data in turn powers other methods that rely on the information in Visual Genome to supervise other tasks \cite{lu2016visual, fukui2016multimodal}. 
This trend of using the possibly noisy information about the data at hand in order to weakly supervise models that look for complex relationships is one of the main motivations for Coral, which is a domain-agnostic method that can combine various sources of weak supervision while modeling the relationship among the sources as well as the primitives they operate over. 

A similar methodology of using interpretable characteristics of images for classification tasks exists in the medical field as well \cite{barb2005knowledge, raicu2009modelling}.
Other techniques such as content-based image retrieval (CBIR) rely on low-level, quantitative features of images, such as tumor texture, obtained from image analysis technique to query similar images in a database of radiology images. 
However, there remains a gap between these content-based primitives and semantic, interpretable primitives that humans use in order to describe these images \cite{kurtz2014combining}.
Coral could combine the two sources of supervision in order to build a system that uses the two methodologies in an optimal manner, while also helping discover correlations among the semantic and quantitative primitives.

\paragraph{Generative Models to Generate Training Labels}
Crowdsourcing is another area where the estimating worker accuracy without hand-labeled data is a well-studied problem \cite{dalvi2013aggregating,joglekar2015comprehensive,zhang2016spectral}. 
Boosting is also a related method that combines the output of multiple weak classifiers to create a stronger classifier \cite{schapire2012boosting}, and recent work has further looked at leveraging unlabeled data \cite{balsubramani2015scalable}.
Co-training is another process that uses a small labeled dataset and a large amount of unlabeled data by selecting two conditionally independent views of the data \cite{blum1998combining}.

\section{Additional Experimental Results}
We present the results from Section~\ref{sec:exp} in a table format here.
We also provide a list of the domain-specific primitives and heuristic functions used in each of our experiments, with a short description of what they mean.

\begin{table}[bthp]
\centering
\caption{Numbers for Figure~\ref{fig:results}(*ed have F1 scores, rest are accuracy). MV: majority vote; Indep: assume heuristics independent; Learn: model learned dependencies; HF-Dep: model heuristic dependencies from static analysis; HF+DSP Dep: also model primitive dependencies; FS: fully supervised}
\label{table:results}
\begin{tabular}{@{}ccccccc@{}}
\toprule
\textbf{Application} & \textbf{MV} & \textbf{Indep} & \textbf{Learn Dep} & \textbf{HF Dep} & \textbf{HF+DSP Dep} & \textbf{FS} \\ \midrule
Visual Genome*       & 42.02       & 46.61          & 46.61              & 49.51           & 49.51               & 50.25       \\
ActivityNet*         & 46.44       & 48.86          & 48.86              & 52.67           & 52.67               & 54.54       \\ \midrule
Bone Tumor           & 65.72       & 67.32          & 67.83              & 70.79           & 70.89               & 67.82       \\
Mammogram            & 63.8        & 67.31          & 68.42              & 67.31           & 68.42               & 69.06       \\ \bottomrule
\end{tabular}
\end{table}

\newpage
\section{Proof of \Cref{thm: scale}}
\thmscale*
\begin{proof}
The proof of \cref{thm: scale} closely follows Theorem 2 of \citet{ratner2016data}.
First, notice that all of the necessary conditions of this theorem are satisfied.
The number of weights to be learned in this theorem is $M$.
In our setting, $M = n + s$.
Notice that if we have at least $\Omega\left[(n + s)\log(n + s)\right]$ unlabeled data points, then we satisfy the conditions of the theorem.
As a result, the bound on the expected parameter error directly follows.
\end{proof}
%
%
%
%
%

\newpage
\begin{table}
\centering
\caption{Domain-specific Primitives for Mammogram Tumor Classification}
\label{table:primitives_1}
  \begin{tabularx}{\textwidth}{llX}
 \toprule
 \textbf{Primitives} & \textbf{Type} & \textbf{Description}\\ 
 \midrule
 \texttt{area} & \texttt{float} & area of the tumor\\ 
 \texttt{diameter} & \texttt{float} & diameter of the tumor\\ 
 \texttt{eccentricity} & \texttt{float} & eccentricity of the image\\
 \texttt{perimeter} & \texttt{float} & perimeter of the tumor\\
 \texttt{max\_intensity} & \texttt{float} & maximum intensity of the tumor\\
 \texttt{mean\_intensity} & \texttt{float} & mean intensity of the tumor\\ 
 \bottomrule
\end{tabularx}
\end{table}

\begin{table}
\centering
\caption{Domain-specific Primitives for Image Querying}
\label{table:primitives_3}
  \begin{tabularx}{\textwidth}{llX}
  \toprule
 \textbf{Primitives} & \textbf{Type} & \textbf{Description}\\ 
 \midrule
 \texttt{person} & \texttt{bool} & image contains a person/man/woman \\
 \texttt{person.position} & \texttt{(float,float)} & coordinates of the person\\
 \texttt{person.area} & \texttt{float} & area of bounsing box of person\\
 \texttt{road} & \texttt{bool} & image contains a road/street\\
 \texttt{car} & \texttt{bool} & image contains car/bus/truck\\  
 \texttt{bike} & \texttt{bool} & image contains a bike/bicycle/cycle\\
 \texttt{bike.position} & \texttt{(float,float)} & coordinates of the bike\\
 \texttt{bike.area} & \texttt{float} & area of bounding box of bike\\ 
 \bottomrule
\end{tabularx}
\end{table}
 
\begin{table}
\centering
\caption{Domain-specific Primitives for Video Classification}
\label{table:primitives_4}
  \begin{tabularx}{\textwidth}{llX}
  \toprule
 \textbf{Primitives} & \textbf{Type} & \textbf{Description}\\ 
 \midrule
 \texttt{person} & \texttt{bool} & image contains a person\\  
 \texttt{ball} & \texttt{bool} & image contains sports ball\\
 \texttt{(person.top,person.bottom)} & \texttt{(float,float)} & top and bottom coordinates of the person\\  
 \texttt{(ball.top,ball.bottom)} & \texttt{(float,float)} & top and bottom coordinates of the ball\\ 
 \texttt{ball.color} & \texttt{(float, float, float)} & R,G,B colors of the ball \\
 \texttt{vertical\_distance} & \texttt{float} & cumulative absolute difference in \texttt{ball.top} values over frames\\
 \bottomrule
\end{tabularx}
\end{table} 
 
\begin{table}
\centering
\caption{Domain-specific Primitives for Bone Tumor Classification}
\label{table:primitives_2}
  \begin{tabularx}{\textwidth}{llX}
  \toprule
 \textbf{Primitives} & \textbf{Type} & \textbf{Description}\\ 
    \midrule
 \texttt{daube\_hist\_164} & \texttt{float} & Daubechies features\\
 \texttt{daube\_hist\_224} & \texttt{float} & Daubechies features\\
 \texttt{daube\_hist\_201} & \texttt{float} & Daubechies features\\
 \texttt{window\_std} & \texttt{float} & \\ 
 \texttt{window\_median} & \texttt{float} & \\ 
 \texttt{scale\_median} & \texttt{float} & \\ 
 \texttt{lesion\_density} & \texttt{float} & quantify edge sharpness along the lesion contour\\
 \texttt{edge\_sharpness} & \texttt{float} & quantify edge sharpness along the lesion contour\\ 
 \texttt{equiv\_diameter} & \texttt{float} & describe the morphology of the lesion\\ 
 \texttt{area} & \texttt{float} & describe the morphology of the lesion\\ 
 \texttt{perimeter} & \texttt{float} & describe the morphology of the lesion\\ 
 \texttt{area\_perimeter\_ratio} & \texttt{float} & ratio of \texttt{area} and \texttt{perimeter}\\ 
 \texttt{shape\_solidity} & \texttt{float} & \\ 
 \texttt{laplacian\_entropy} & \texttt{float} & Laplacian energy features\\ 
 \texttt{sobel\_entropy} & \texttt{float} & Sobel energy features\\
 \texttt{glcm\_contrast} & \texttt{float} & capture occurrence of gray level pattern within the lesion\\
 \texttt{glcm\_homogeneity} & \texttt{float} & capture occurrence of gray level pattern within the lesion\\
 \texttt{histogram\_egde} & \texttt{float} & quantify edge sharpness along the lesion contour\\
 \texttt{mean\_diff\_in\_out} & \texttt{float} & \\
 \bottomrule
\end{tabularx}
\end{table}

\begin{table}
\centering
\caption{Heuristic Functions for Mammogram Tumor Classification}
\label{table:lfs_1}
\begin{tabularx}{\textwidth}{lXX}
 \toprule
 \textbf{Name} & \textbf{Heuristic Function} & \textbf{Description}\\ 
 \midrule
 \texttt{hf\_area} & 1 if \texttt{area}$>=$100000; -1 if \texttt{area}$<=$30000 & Large tumor area indicates malignant tumors\\  
 \texttt{hf\_diameter} & 1 if \texttt{diameter}$>=$400; -1 if \texttt{diameter}$<=$200 & High diameter indicates malignant tumors\\ 
 \texttt{hf\_eccentricity} & 1 if \texttt{eccentricity}$<=$0.4; -1 if \texttt{eccentricity}$>=$0.6 & Low eccentricity indicates malignant tumors\\
 \texttt{hf\_perimeter} & 1 if \texttt{perimeter}$>=$4000; -1 if \texttt{perimeter}$<=$2500 & High perimeter indicates malignant tumors\\
 \texttt{hf\_max\_intensity} & 1 if \texttt{max\_intensity}$>=$70000; -1 if \texttt{max\_intensity}$<=$50000 & High maximum intensity indicates malignant tumors\\
 \texttt{hf\_mean\_intensity} & 1 if \texttt{mean\_intensity}$>=$45000; -1 if \texttt{mean\_intensity}$<=$30000& High mean intensity indicates malignant tumors\\
 \bottomrule  
\end{tabularx}
\end{table}

\begin{table}
\centering
\caption{Heuristic Functions for Image Querying}
\label{table:lfs_3}
  \begin{tabularx}{\textwidth}{lXX}
 \toprule
 \textbf{Name} & \textbf{Heuristic Function} & \textbf{Description}\\ 
 \midrule
 \texttt{hf\_street} & 1 if \texttt{person} and \texttt{road}; -1 if \texttt{person} and \texttt{!road}; 0 otherwise & Indicates if a street is present when person is present, doesn't assign a label when person is not present\\
 \texttt{hf\_vehicles} & 1 if \texttt{person} and \texttt{car}; -1 if \texttt{person} and \texttt{!car}; 0 otherwise & Indicates if a car is present when person is present, doesn't assign a label when person is not present\\
 \texttt{hf\_positions} & -1 if \texttt{!bike} or \texttt{!person}; 1 if (\texttt{person.position}-\texttt{bike.position})$<=$1; 0 otherwise & Indicates if person and bike are close in the image\\
 \texttt{hf\_size} & 0 if \texttt{!bike} or \texttt{!person}; -1 if (\texttt{person.area}-\texttt{bike.area})>=1000; 1 otherwise & Indicates if the difference in area of the bike and person is less than a threshold\\
 \texttt{hf\_number} & -1 if \texttt{!bike} or \texttt{!person}; 1 if \texttt{num\_persons}$=$\texttt{num\_bikes}; -1 if \texttt{num\_persons}$=$\texttt{num\_bikes}; 0 otherwise & Indicates if number of persons and bikes are equal\\ 
 \bottomrule 
\end{tabularx}
\end{table}

\begin{table}
\centering
\caption{Heuristic Functions for Video Classification}
\label{table:lfs_4}
  \begin{tabularx}{\textwidth}{lXX}
  \toprule
 \textbf{Name} & \textbf{Heuristic Function} & \textbf{Description}\\ 
 \midrule
 \texttt{hf\_person\_ball} & 1 if \texttt{person} and \texttt{ball}; -1 otherwise & Indicates if person and sports ball were present in any frame of the video\\  
 \texttt{hf\_distance} & -1 if \texttt{person.top}-\texttt{ball.bottom}$>=$2; 1 if \texttt{person.top}-\texttt{ball.bottom}$<=$1; 0 otherwise & Indicates if the distance between person and sports ball is less than a threshold \\ 
 \texttt{hf\_ball\_color} & 1 if \texttt{ball} and \texttt{ball.color}-\texttt{basketball.color} $<=$ 80; -1 otherwise & Indicates if the color of the ball is similar to the color of a basketball\\
 \texttt{hf\_temporal} & 1 if \texttt{vertical\_distance} $>=$15; -1 otherwise & Indicates if sufficient vertical distance was covered by the ball over frames\\
 \bottomrule
\end{tabularx}
\end{table}

\begin{table}
\centering
\caption{Heuristic Functions for Bone Tumor Classification}
\label{table:lfs_2}
  \begin{tabularx}{\textwidth}{lX}
 \toprule
 \textbf{Name} & \textbf{Heuristic Function} \\ 
 \midrule
 \texttt{hf\_daube} & 1 if \texttt{histogram\_164} $<$  0.195545 and \texttt{histogram\_224} $<$ -0.469812; -1 if \texttt{histogram\_164} $<$  0.195545 and \texttt{histogram\_201} $<$ 0.396779; 1 if \texttt{histogram\_164} $<$  0.195545; -1 otherwise\\  
 \texttt{hf\_edge} & -1 if \texttt{window\_std} $<$ -0.0402606; -1 if \texttt{window\_median} $<$ -0.544591; -1 if \texttt{scale\_median} $<$ -0.512551; 1 otherwise\\ 
 \texttt{hf\_lesion} & -1 if \texttt{lesion\_density} $<$ 0; 1 if \texttt{lesion\_density} $>$ 1 and \texttt{edge\_sharpness} $<$ 0; 1 if \texttt{lesion\_density} $>$ 5 and \texttt{edge\_sharpness} $>$ -1; 0 otherwise\\ 
 \texttt{hf\_shape} & -1 if \texttt{equiv\_diameter} $<$ -0.3; -1 if \texttt{equiv\_diameter} $>$ 0 and \texttt{area\_perimeter\_ratio} $<$ 0.5 and \texttt{shape\_solidity} $<$ 0.1; 1 if \texttt{equiv\_diameter} $>$ 0 and \texttt{area\_perimeter\_ratio} $<$ 0.5 and \texttt{shape\_solidity} $>$ 0.75; 1 if \texttt{equiv\_diameter} $>$ 0 and \texttt{area\_perimeter\_ratio} $>$ 1; 0 otherwise\\ 
 \texttt{hf\_sobel\_laplacian} & -1 if \texttt{laplacian\_entropy} $<$ 0.2; 1 if \texttt{laplacian\_entropy} $>$ 0.4 and \texttt{sobel\_entropy} $<$ -0.75; 1 if \texttt{laplacian\_entropy} $>$ 0.4 and \texttt{sobel\_entropy} $>$ -0; 0 otherwise\\ 
 \texttt{hf\_glcm} & -1 if \texttt{gclm\_contrast} $<$ 0.15 and \texttt{gclm\_homogeneity} $<$ 0; 1 if \texttt{gclm\_contrast} $<$ 0.15 and \texttt{gclm\_homogeneity} $>$ 0.5; -1 if \texttt{gclm\_contrast} $>$ 0.25; 0 otherwise\\ 
 \texttt{hf\_first\_order} & -1 if \texttt{histogram\_egde} $<$ 0.5; 1 if \texttt{histogram\_egde} $>$ -0.3 and \texttt{mean\_diff\_in\_out} $<$ -0.75; 1 if \texttt{histogram\_egde} $>$ -0.3 and \texttt{mean\_diff\_in\_out} $>$ -0.5; 0 otherwise\\  
 \bottomrule 
\end{tabularx}
  \phantom{.}\\
\raggedright We omit the descriptions of heuristic functions for bone tumor classification due to their complexity.
\end{table}

\end{document}